\documentclass{article}

\usepackage{arxiv}

\usepackage{times}
\usepackage{epsfig}
\usepackage{graphicx}
\usepackage{amsmath}
\usepackage{amssymb}
\usepackage{comment}
\usepackage{amsthm}
\usepackage{caption}
\usepackage{subcaption}
\usepackage{multirow}
\usepackage{makecell}
\usepackage{xfrac}
\usepackage{times}
\usepackage{epsfig}
\usepackage{graphicx}
\usepackage{amsmath}
\usepackage{amssymb}

\newtheorem{prop}{Proposition}

\usepackage[breaklinks=true,bookmarks=false]{hyperref}

\begin{document}

\title{Mini-batch graphs for robust image classification}

\author{Arnab Kumar Mondal\thanks{Equal contribution.}\\
McGill University\\
{\tt\small arnab.mondal@mila.quebec}
\and
Vineet Jain\footnotemark[1]\\
Carnegie Mellon University\\
{\tt\small vineetja@cs.cmu.edu}

\and
Kaleem Siddiqi\\
McGill University\\

{\tt\small siddiqi@cim.mcgill.ca}
}

\renewcommand{\undertitle}{}

\maketitle

\begin{abstract}
 Current deep learning models for classification tasks in computer vision are trained using mini-batches. In the present article, we take advantage of the relationships between samples in a mini-batch, using graph neural networks to aggregate information from similar images. This helps mitigate the adverse effects of alterations to the input images on classification performance. Diverse experiments on image-based object and scene classification show that this approach not only improves a classifier's performance but also increases its robustness to image perturbations and adversarial attacks.
Further, we also show that mini-batch graph neural networks can help to alleviate the problem of mode collapse in Generative Adversarial Networks.
\end{abstract}

\section{Introduction}
In recent years, supervised deep learning has had wide success in computer vision problems, including in object and scene classification \cite{krizhevsky2012, places2017tpami} and image segmentation \cite{Benenson_2019_CVPR, Isaacs_2020_CVPR}. Models such as residual networks \cite{he2016deep} have become standard and are now used as encoders to learn image based representations. In most of these settings, the training data is divided into mini-batches to adjust to limitations in computational and memory resources.
Within a particular mini-batch, the input images may have varying degrees of similarity between them. Exploiting this variability during the feature encoding stage has the potential to improve the performance of downstream computer vision tasks.

Motivated by this idea, different approaches have been proposed to take advantage of the relationships between samples in a mini-batch for computer vision tasks, and in particular for image-based classification. These approaches all explicitly
encourage the embeddings in the feature space to be close to one another when the underlying images are similar, by using an extra similarity based loss term. As an example, in \cite{chen2020simple} contrastive learning is used so that different augmentations of the same image within a mini-batch have a high degree of pair-wise similarity.
In \cite{khosla2020supervised} this approach is extended to a supervised setting, such that the embeddings for instances within the same class are nearby.
In a similar vein, in \cite{wang2020affinity} the learning of representations is supervised by increasing the affinity between mini-batch samples that belong to the same class. 

In the present article, we propose a more direct approach 
to information aggregation across each mini-batch of images, which is to use graph neural networks (GNNs) for this purpose. 
Our approach is based on the construction of a graph from each mini-batch of samples, which allows information to be pooled across those with similar features, using graph convolution operations. As such, the requirement that similar images have similar embeddings is {\em implicit}, in that no additional loss term has to be optimized.
This allows for the aggregation of features during training in a manner that adjusts dynamically to each particular mini-batch ensemble of images. A perturbation analysis explains how this, in turn, affords a degree of robustness to input image alterations and adversarial attacks. We also show the connection of our mini-batch graph neural networks to mini-batch discrimination, a technique to mitigate the mode collapse problem faced during training of Generative Adversarial Networks (GANs) \cite{goodfellow2014generative}.


 
Our experiments show a consistent improvement over the baseline and other related approaches across multiple architectures and datasets for the particular task of image classification. They also reveal the robustness of our proposed model against input perturbations and adversarial attacks. Our proposed mini-batch graph learning method causes little computational overhead since it introduces only a small number of additional parameters for feature aggregation. Since the method is implemented as a modular layer (Figure \ref{fig:MBGNN}) and training in mini-batches is not specific to image classification, with minor modifications, it can be used for other vision tasks including segmentation \cite{Benenson_2019_CVPR, Isaacs_2020_CVPR}, region proposal generation \cite{ren2015faster} and relationship modeling \cite{zhang2017relationship}. 

\section{Related work}

The modeling of relationships between samples in a mini-batch has already shown promise in computer vision tasks. In recent work, \cite{wang2020affinity}, the learning of affinity between samples is supervised by optimizing an affinity mass loss. Here the pairwise affinity between {\em all} samples in the mini-batch is considered, and the loss function encourages the model to increase the affinity between samples belonging to the same class. 


A different approach to exploit relationships between samples while training in mini-batches is to use supervised contrastive learning \cite{khosla2020supervised}. Here, the normalized embeddings from samples in the same class are encouraged to be closer to one another than to the embeddings of samples from different classes. This approach is related to another self-supervised contrastive method \cite{chen2020simple}, where a model is trained to identify samples in a mini-batch that are different augmentations of the same image. These methods demonstrate improvement in image classification performance over standard networks.

Motivated by the success of the above methods, in the present article we propose and evaluate a more flexible approach, which is to use graph neural networks (GNNs) during mini-batch training to aggregate features from similar samples. Distinct from affinity graphs \cite{wang2020affinity} and contrastive learning \cite{chen2020simple,khosla2020supervised}, the use of GNNs in our approach encourages similar images to have similar embeddings in an implicit manner, while removing the requirement of the need to optimize an additional loss term.
A GNN learns node representations that reflect the local neighborhood by aggregating information across nodes.

GNNs were first proposed as deep learning architectures on graph-structured data in \cite{gori2005new,scarselli2008graph} and have since been extended to include convolution operation on graphs in \cite{bruna2013spectral,henaff2015deep, defferrard2016convolutional} or to combined locally connected regions in graphs\cite{pmlr-v48-niepert16}. 
The use of GNNs for semi-supervised classification was proposed in \cite{kipf2016semi}, following which several different variants of GNN models have been developed: Graph Attention Networks (GATs) \cite{velickovic2018graph}, models to process graphs with edge information \cite{gong2019exploiting,wang2019dynamic} and GNNs that work under low homophily \cite{zhu2020generalizing}.

GNNs have also been successfully applied to several other computer vision problems, including few-shot learning \cite{satorras2018few,rodriguez2020embedding} and semi-supervised learning \cite{zhuang2019local}. In \cite{satorras2018few}, a GNN is used to propagate label information from labeled samples in the support towards the unlabeled query image. In contrast, in \cite{rodriguez2020embedding,zhuang2019local}, a fixed graph is used to propagate embedding and label information, respectively.



\begin{figure*}[hbt!]
\begin{center}
    \includegraphics[width=0.9\textwidth]{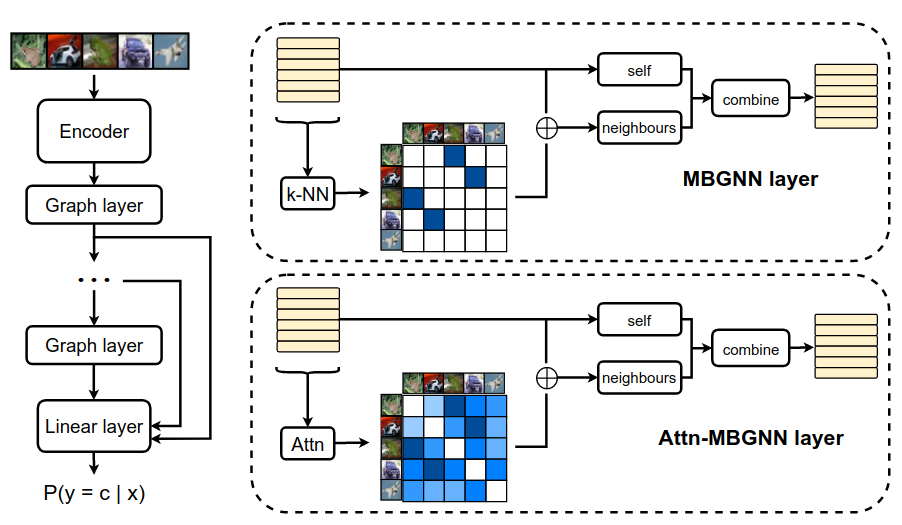}
    \caption{An illustration of the proposed MBGNN model for mini-batch learning for image classification. Encoder representations are used to create an adjacency matrix based on k-nearest neighbours. The adjacency matrix defines a graph on a mini-batch of representations, which is then processed using graph convolution operations, as detailed in Section \ref{sec:MBGNNs}.}
    \label{fig:MBGNN}
\end{center}
\end{figure*}

\section{Mini-batch Graph Neural Networks}
\label{sec:MBGNNs}

\subsection{Proposed Method}
Our network has two components: a feature encoder and a mini-batch graph neural network (MBGNN), as illustrated in Fig. \ref{fig:MBGNN}. We obtain the encoder $f_\theta(\cdot)$ by removing the final layer of a standard residual network, such as Resnet-50. Consider a typical training setup which takes a mini-batch of samples as input for a downstream classification task. We denote the input samples in a mini-batch as $\mathcal{X}=\{x_1, x_2, ..., x_B\}$, where $B$ is the batch size. The encoder provides representations for each sample, $h_i^{(0)} = f_\theta(x_i), \forall i \in B$. The MBGNN induces a graph on the set of encoded representations and processes them using graph convolution operations.
We denote the representations in the $l^\text{th}$ layer by the set $\mathcal{H}^{(l)}=\{h_1^{(l)}, h_2^{(l)}, ..., h_B^{(l)}\}$. To dynamically induce a graph on $\mathcal{H}^{(l)}$, we obtain the adjacency matrix $A^{(l)}$ by computing pairwise cosine similarity between representations and consider the top $k$ similar representations for each sample as its neighbours, removing self connections. The extent of the neighbourhood for each node can be controlled with the parameter $k$. The layer-wise propagation rule of the MBGNN in vector form, is given by,
\begin{equation}
\begin{split}
\Bar{H}^{(l)} &= H^{(l)}W^{(l)}+ b^{(l)} \\ 
H^{(l+1)} &= \sigma\big(\text{combine}\big(\Bar{H}^{(l)}, (1/k)A^{(l)} \Bar{H}^{(l)}\big)\big). 
\end{split}
\label{eq_vector_DMBGNN}
\end{equation}
where $h^{(l)}_i $ is stacked row-wise to form $H^{(l)}$, $W^{(l)}$ is the weight matrix, $b^{(l)}$ is the bias, and $\sigma(\cdot)$ is a non-linear function, usually ReLU.
In equation \eqref{eq_vector_DMBGNN}, $f_\textrm{self}^{(l)} = \Bar{H}^{(l)}$ contains the `self' information for each node, and $f_\textrm{neigh}^{(l)} = (1/k)A^{(l)} \Bar{H}^{(l)}$ contains the `neighbour' information, since it is based on the average of the encoded representations of the neighbors of each node, as reflected by the adjacency matrix $A^{(l)}$. 

There are several methods to combine the self and neighbour information in the $\text{combine}(\cdot, \cdot)$ function, which are explained below.

\underline{Concatenation}: Here the self features, $f_\textrm{self}^{(l)}$ and neighbour features $f_\textrm{neigh}^{(l)}$ are concatenated in the representation dimension. This allows the network the flexibility to learn separate weights for both.

\underline{Weighted Addition}: This is a convex combination of self and neighbour features, which forces the network to use neighbour information. For $\alpha \in [0,1]$, we have,
\begin{equation}
    \text{combine}(f_\textrm{self}^{(l)},f_\textrm{neigh}^{(l)}) = \alpha f_\textrm{self}^{(l)} + (1-\alpha)f_\textrm{neigh}^{(l)}.
\end{equation}
This reduces to a standard GCN formulation if we set $\alpha=1/{k+1}$.

\underline{Drop Features}: We propose a different method in to mitigate the effect of co-adaptation and to make the neighbours contribute meaningful information. During training, we drop the self features with probability $p$ and the neighbour features with probability $1-p$. We then make the testing phase deterministic by taking the expected output feature, which leads to a similar expression as the sum combination: 
\begin{equation}
    \text{combine}(f_\textrm{self}^{(l)},f_\textrm{neigh}^{(l)}) = p f_\textrm{self}^{(l)} + (1-p)f_\textrm{neigh}^{(l)}.
\end{equation}
We explore these different methods in our experiments.

\paragraph{Attention MBGNN}The model can also be made more expressive by using an attention mechanism while aggregating from neighbours, which we refer to as an Attn-MBGNN (see Fig. \ref{fig:MBGNN}). This is done by changing the calculation of the adjacency matrix $\mathcal{A}^{(l)}$ to incorporate an attention coefficient between nodes. Let $\alpha^{n(l)}_{ij}$ denote the attention coefficient of node $j$ to $i$ for the $n$-th attention head, which can be computed as,

\begin{equation}
    \alpha_{ij}^{n(l)}=\frac{\exp\big(\phi\big(\psi(W_{n}^{(l)} h_i^{(l)}, W_{n}^{(l)}h_j^{(l)})\big)\big)}{\sum_{m \in \mathcal{N}(i)} \exp\big(\phi\big(\psi(W_{n}^{(l)}h_i^{(l)}, W_{n}^{(l)}h_{m}^{(l)})\big)\big)}.
\label{eq_attention}
\end{equation}
Here $\phi(\cdot)$ is a neural network, $W_{n}^{(l)}$ is a trainable matrix and $\psi$ is absolute difference. This is similar to the attention coefficient used in GAT \cite{velickovic2018graph}. To form the weighted adjacency matrix, we first follow the same process as for MBGNN by considering top $k$ similar features based on cosine similarity to get the neighbourhood $\mathcal{N}(i)$ for each node $i$. To get the weighted adjacency matrix for the $n$-th head we set $\mathcal{A}^{n(l)}_{ij} = \alpha^{n(l)}_{ij}, \forall j \in \mathcal{N}(i)$ and  $\mathcal{A}^{n(l)}_{ij} = 0, \forall j \notin \mathcal{N}(i)$. We also remove self-connections by setting $\mathcal{A}^{n(l)}_{ii}=0, \forall i$. 
The vectorized layer-wise propagation rule of the Attn-MBGNN with $N$ attention heads then becomes
\begin{equation}
\begin{split}
\Bar{H}^{(l)} &= H^{(l)}W^{(l)}+ b^{(l)} \\ 
H^{(l+1)} &= \sigma\big(\text{combine}\big(\Bar{H}^{(l)},\frac{1}{N}\sum_{n=1}^{N} \mathcal{A}^{n(l)} \Bar{H}^{(l)}\big)\big).
\end{split}
\label{eq_vector_DAMBGNN}
\end{equation}

Once the intermediate representations (upto $L$ layers) are obtained either using MBGNN or Attn-MBGNN, we get the final logits for the $i$-th input in the batch by,
\begin{equation}
    \ell_i = W_{final}(h_i^{(1)} \,\|\, h_i^{(2)} \,\|\, \ldots \,\|\, h_i^{(L)}),
\end{equation}
where $\|$ denotes concatenation in the feature dimension. This captures the local and global information separately and takes a weighted combination. This design choice has been shown to increase the representation power of GNNs \cite{xu2018representation}, by leveraging different neighbourhood ranges to better enable structure-aware representation. We can now compute the class probabilities by taking the softmax, $ p(y_i=k | x_i) = \frac{\exp(\ell_{k})}{\sum_{j=1}^{C} \exp(\ell_{j})}$ where $C$ is the total number of classes. We use cross entropy loss to train the encoder and the MBGNN model end-to-end,
\begin{equation}
    \mathcal{L} = - \sum_{i=1}^{B} \sum_{k=1}^{C} \mathbb{1}_{y_i = k} \log p(y_i=k | x_i). 
\end{equation}

\paragraph{Evaluation settings} In the \textit{transductive} setting, we predict the label for a single test image within a mini-batch graph consisting of one test image and multiple training images. In the \textit{inductive} setting, we construct a graph on a full mini-batch consisting of test images and predict labels for all samples. We include results under both the settings.

\subsection{Robustness to Input Perturbations}
We now show how our MBGNN provides robustness in the face of perturbations in the input, which we have argued is a desirable property for downstream computer vision tasks.

\begin{prop}
\label{prop:robustness}
Consider a neural network comprised of an encoder and a fully connected layer, $g_{sup}(\cdot)$, and a MBGNN neural network consisting of an encoder followed by a graph convolutional layer where each node has $k$ neighbours, by $g_{mbgnn}(\cdot)$. For transductive prediction, consider an input sample $x$, with some perturbation $\Delta x$. Let the associated perturbations in logits be $\Delta y_{sup} = g_{sup}(x+\Delta x) - g_{sup}(x)$ and $\Delta y_{mbgnn} = g_{mbgnn}(x+\Delta x) - g_{mbgnn}(x)$. Then, $\Delta y_{mbgnn} = \frac{1}{k+1} \Delta y_{sup}$.
\end{prop}
\begin{proof}
Let the encoder $e(\cdot)$ output a vector $e(x)$ for a given image $x$, and denote a given batch of samples as $\{x_1, x_2, ..., x_B\}$. 
For the standard supervised model, denoting the weight matrix of the final layer as $W$, we get the perturbation in the final pre-softmax logits, for perturbation in input $x_B$, as $\Delta y_{sup} = W^Te(x_B+\Delta x_B) - W^Te(x_B) = W^T[e(x_B+\Delta x_B) - e(x_B)]$. Now, consider an MBGNN model in which each sample is connected to $k$ other samples in the mini-batch and has the same weight matrix $W$. For an MBGNN with self connections, using the standard GCN update rule, we have $\Delta y_{mbgnn} = \frac{1}{k+1}W^T[e(x_{B}+\Delta x_{B}) + \sum_{j\in \mathcal{N}(B)}e(x_j)] - \frac{1}{k+1}W^T[e(x_{B}) + \sum_{j\in \mathcal{N}(B)}e(x_j)] = \frac{1}{k+1}W^T[e(x_{B}+\Delta x_{B}) - e(x_{B})] = \frac{1}{k+1} \Delta y_{sup}$  
\end{proof}

The above proposition effectively states that for any perturbation in the input, the corresponding perturbation in the output is inversely proportional to the number of neighbours, for each node in the mini-batch graph, when using MBGNNs as opposed to standard networks. Similarity based aggregation aids in transductive inference, where a prediction is made for a single corrupted test image within a mini-batch of randomly sampled uncorrupted training set images. In Section \ref{sec:robustness} we carry out experiments to verify this property of robustness to image perturbations.  

We also test the robustness of the model against black-box adversarial attacks. These adversaries craft perturbations which cause the model to classify legitimate looking input images incorrectly. Black-box adversaries do not have access to the model parameters and the gradients, and must query the model to observe the output class probabilities. They query the model repeatedly with a chosen image, but perturbing it with each iteration, based on the results of the previous query.
A model which has a lower attack success rate and/or requires a higher number of queries on average against an adversary before a successful attack, is considered more robust. We test the MBGNN model against two recently proposed and popular black-box adversarial attack methods, simBA \cite{Guo2019SimpleBA} and Bandits-TD \cite{ilyas2018prior}. We choose these two methods since they use different methodologies - simBA uses local search in order to craft adversarial perturbations whereas Bandits-TD estimates the gradient by repeatedly querying the model to create the adversarial input. The results of these experiments are also provided in Section \ref{sec:robustness}.

\subsection{Connection to Mini-batch Discrimination}

Generative Adversarial Networks (GANs), first introduced in \cite{goodfellow2014generative}, are a family of generative models that are used in several computer vision tasks including high resolution image generation\cite{wang2018high}, image super-resolution\cite{ledig2017photo}, domain adaptation\cite{isola2017image,zhu2017unpaired} and image compression \cite{agustsson2019generative}. 
GANs suffer from the problem of mode collapse, where the generated samples belong to a few modes in the dataset while still being successful at fooling the discriminator. This leads to a lack of diversity in the generated samples. One way to mitigate mode collapse in GANs is to use a technique known as {\em mini-batch discrimination} \cite{DBLP:journals/corr/SalimansGZCRC16}. Here, instead of the discriminator being required to label individual samples as fake or real, it discriminates between an entire mini-batch of generated or real samples. 

As it turns out, our proposed Attn-MBGNN can be interpreted as an extension of mini-batch discrimination to the classification task. To establish this connection, we present a variation of our MBGNN, which we refer to as a Mode Collapse MBGNN (MC-MBGNN). Rather than aggregating node features weighted by attention, as in the case of Attn-MBGNN layers, we aggregate the edge features in this model. Note that these edge features capture similarity and can be interpreted as the unnormalized attention values.  

In MC-MBGNN, we modify the discriminator network by extracting features from real/fake images using the encoder, $f_\theta(\cdot)$, and denote them by $\{h_1, h_2, ..., h_{B}\}$ where $h_i = f_\theta(x_i)$ and $B$ is the batch size. We induce complete graphs for both the batch of generated and real samples, and process them separately. 
The output for the $n$-th aggregated edge feature for the $i$-th sample in the mini-batch, which is computed in a manner similar to aggregating unnormalized attention weights for the $n$-th head in Attn-MBGNN, is given by:
\begin{equation}
    \Bar{h}_{i}^{n} = \sum_{j=1}^{B}\phi(\psi(W_nh_i,W_nh_j)),
\end{equation}
where $\phi(\cdot)$ is a neural network, $W_n$ is a trainable matrix, and $\psi$ can either be concatenation or absolute difference. We can then concatenate the aggregated edge features with the independent node features and use a final layer to get the $i$-th scalar output: $ o_i = \sigma(W_{final}(h_i\,\|\,\Bar{h}_{i}^1\,\|\,....\,\|\,\Bar{h}_{i}^N)).$ where $\sigma(\cdot)$ is the sigmoid function. 
If we restrict $\psi(\cdot)$ to absolute difference, $\phi(\cdot)$ to a fixed $\exp{}(-x)$ function and take all the $\Bar{h}_{i}^n$ as scalars then the model reduces to standard mini-batch discrimination \cite{DBLP:journals/corr/SalimansGZCRC16}. This shows how MBGNNs are connected to mini-batch discrimination. Our MC-MBGNN model is more expressive than mini-batch discrimination, and can significantly mitigate mode collapse in GANs, as demonstrated by the experiments in Section \ref{sec:modecollapse}.

\section{Experiments}

\subsection{Datasets}
We perform experiments on two standard computer vision object classification datasets, CIFAR-10 and CIFAR-100. In order to expand the scope of the experiments to include scene classification, which is more complex, we also test our models on the MIT 67 scene dataset. For our GAN experiments we use the CIFAR-10 and the CelebA datasets.\\
\textbf{CIFAR-10} consists of 60,000 colour images in 10 classes, with 6,000 images per class. We use the standard split with 50,000 training images and 10,000 test images. Each image is 32 $\times$ 32.\\
\textbf{CIFAR-100} is similar to CIFAR-10, except that it has 100 classes, which significantly increases the difficulty of the classification task. Each class contains 600 images with 500 training images and 100 testing images per class. Each image is 32 $\times$ 32.\\
\textbf{MIT 67} contains indoor scene images belonging to 67 categories, with a total of 15620 images. The number of images varies across categories, but there are at least 100 images per category. For our experiments we reduce the image size by approximately a factor of 10 in each dimension to 64 $\times$ 64, to ease the computational resources. As a consequence the scene classification task becomes harder.\\
\textbf{CelebA} is a large-scale face attributes dataset containing 10,177 identities with 202,599 colour face images in total. Each image is 64$\times$64.

\begin{table*}[]
\small
\begin{tabular}{|l|c|c|c|c|c|c|}
\hline
 & \multicolumn{2}{c|}{\bf CIFAR 10} & \multicolumn{2}{c|}{\bf CIFAR 100} & \multicolumn{2}{c|}{\bf MIT 67} \\ \cline{2-7}
 \bf Model & \bf Inductive     & \bf Transductive  & \bf Inductive    & \bf Transductive    & \bf Inductive   & \bf Transductive  \\ \hline
Supervised vanilla           & \multicolumn{2}{c|}{$94.69 \scriptstyle{\pm 0.17}$} & \multicolumn{2}{c|}{$74.48 \scriptstyle{\pm 0.25}$} &  \multicolumn{2}{c|}{$64.48 \scriptstyle{\pm 0.27}$}             
\\
Supervised contrastive \cite{khosla2020supervised}      & \multicolumn{2}{c|}{$94.85 \scriptstyle{\pm 0.13}$}  &\multicolumn{2}{c|}{$74.80 \scriptstyle{\pm 0.22}$} &         \multicolumn{2}{c|}{$65.10 \scriptstyle{\pm 0.16}$}        
\\
Affinity supervision \cite{wang2020affinity}      & \multicolumn{2}{c|}{$94.45 \scriptstyle{\pm 0.35}$}  &\multicolumn{2}{c|}{$74.50 \scriptstyle{\pm 0.59}$} & \multicolumn{2}{c|}{$64.60 \scriptstyle{\pm 0.30}$}        
\\ \hline
MBGNN (concat)            &  $95.19 \scriptstyle{\pm 0.23}$  &   $95.21 \scriptstyle{\pm 0.21} $  &  $75.22 \scriptstyle{\pm 0.17}$   &      $75.15 \scriptstyle{\pm 0.20}$  &  $65.80 \scriptstyle{\pm 0.21}$       &  $65.81 \scriptstyle{\pm 0.19}$  \\
MBGNN (sum)               & $95.02 \scriptstyle{\pm 0.19} $  &   $95.02 \scriptstyle{\pm 0.13}$   &      $75.18 \scriptstyle{\pm 0.21}$   &    $75.20 \scriptstyle{\pm 0.15}$   & $65.68 \scriptstyle{\pm 0.17}$       &  $65.70 \scriptstyle{\pm 0.18}$    \\
MBGNN (dropfeat)          & $\bf 95.24 \scriptstyle{\pm 0.19}$   &   $\bf 95.22 \scriptstyle{\pm 0.25}$  & $75.21 \scriptstyle{\pm 0.17}$  &     $75.25 \scriptstyle{\pm 0.19}$  &      $65.82 \scriptstyle{\pm 0.18}$       &  $65.84 \scriptstyle{\pm 0.20}$
\\ \hline
Attn-MBGNN (concat)       & $95.14 \scriptstyle{\pm 0.14}$   &   $95.12 \scriptstyle{\pm 0.21}$  & $75.16 \scriptstyle{\pm 0.22}$  &     $75.19 \scriptstyle{\pm 0.25}$ &  $\bf  65.86 \scriptstyle{\pm 0.15}$       & $\bf 65.87 \scriptstyle{\pm 0.18}$
\\
Attn-MBGNN (sum)          & $94.95 \scriptstyle{\pm 0.18}$   &   $94.98 \scriptstyle{\pm 0.15}$   & $75.23 \scriptstyle{\pm 0.15}$  &     $75.25 \scriptstyle{\pm 0.17}$ & $65.72 \scriptstyle{\pm 0.19}$       &  $65.73 \scriptstyle{\pm 0.17}$
\\
Attn-MBGNN (dropfeat)     & $95.05 \scriptstyle{\pm 0.25} $    &  $95.06 \scriptstyle{\pm 0.22}$  & $\bf 75.29 \scriptstyle{\pm 0.11}$  &    $\bf 75.26 \scriptstyle{\pm 0.18}$ &   $\bf 65.86 \scriptstyle{\pm 0.22}$       &  $65.85 \scriptstyle{\pm 0.20}$        
\\ \hline
\end{tabular}
\caption{Image classification results using a Resnet-50 encoder. The results are shown with 95\% confidence intervals over 5 runs. The architectures are trained using a batch size of 256 and with $k$ = 16 for CIFAR-10, and $k$ = 4 for CIFAR-100 and MIT67. We provide results for different combine modes of our single layered mini-batch graph based models (rows 4-9). Additional results using a Wide Resnet-28-10 encoder are in the supplementary material.}
\label{table:supervised_resnet}
\end{table*}

\begin{figure*}[hbt!]
\begin{center}
    \begin{subfigure}[b]{0.24\textwidth}
        \centering
        \includegraphics[width=\textwidth]{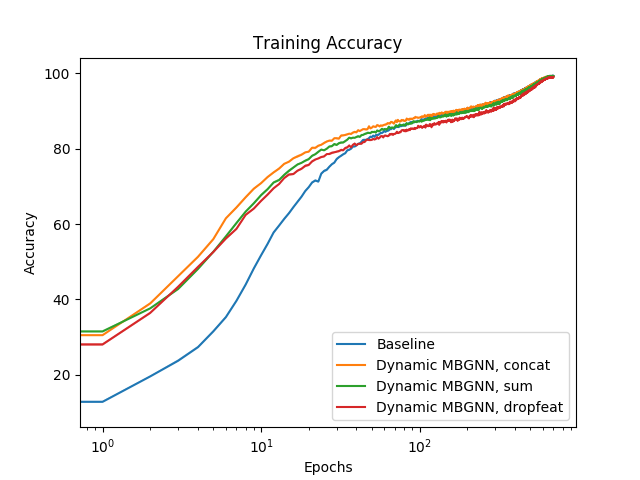}
        \caption{}
    \end{subfigure}
    \hfill
    \begin{subfigure}[b]{0.24\textwidth}
        \centering
        \includegraphics[width=\textwidth]{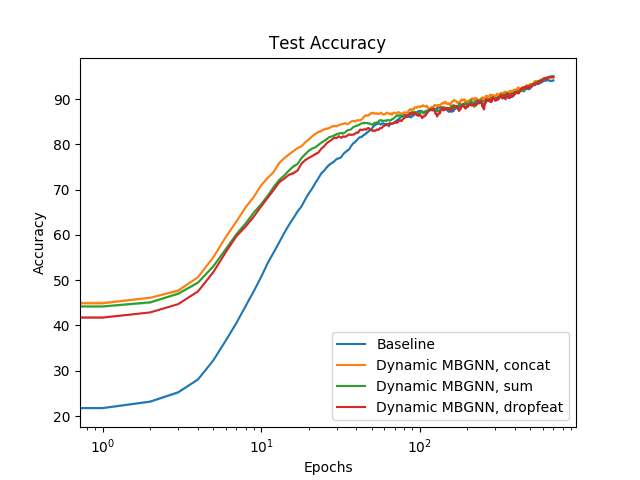}
        \caption{}
    \end{subfigure}
    \hfill
    \begin{subfigure}[b]{0.24\textwidth}
        \centering
        \includegraphics[width=\textwidth]{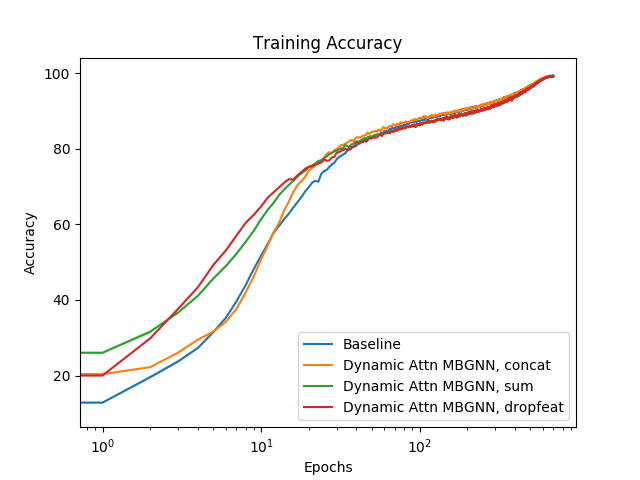}
        \caption{}
    \end{subfigure}
    \hfill
    \begin{subfigure}[b]{0.24\textwidth}
        \centering
        \includegraphics[width=\textwidth]{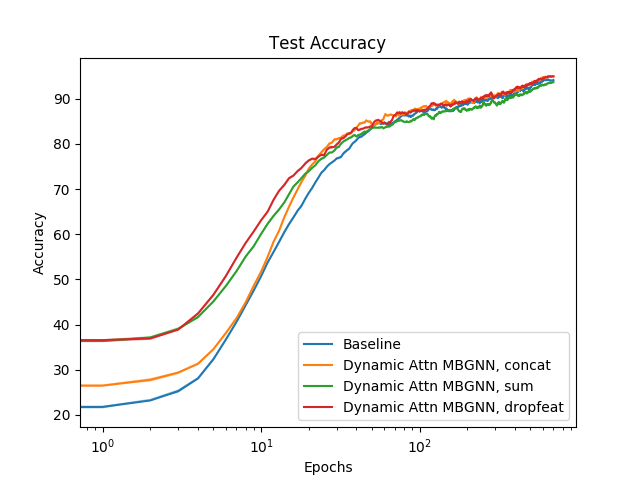}
        \caption{}
    \end{subfigure}
    \caption{We encourage the reader to zoom-in on the PDF. Plots showing training and test accuracy versus epochs, for different models on the CIFAR-10 dataset using ResNet-50 encoder. The epochs axis is on a log scale. (a) training accuracy for MBGNNs, (b) test accuracy for MBGNNs, (c) training accuracy for attention MBGNNs, (d) test accuracy for attention MBGNNs.}
    \label{fig:train_plot}
\end{center}
\end{figure*}

\subsection{Image classification results} 
\label{sec:image_classification}
\textbf{Baselines.} We use a ResNet-50 model trained using cross-entropy loss as the first baseline. We also reproduce the results of \cite{khosla2020supervised} with a batch size of 256 and reproduce results of \cite{wang2020affinity} with our ResNet-50 baseline. We focus on relative improvement between our proposed model and these baselines.

For our proposed model, we use a ResNet-50 network without the final layer as an encoder and add an MBGNN with different combine methods. The entire model is trained end-to-end using cross-entropy loss. We provide the results for both MBGNN and Attn-MBGNN (with 95\% confidence intervals over 5 runs) with all the combine options in Table \ref{table:supervised_resnet}, under both the inductive and the transductive settings. Our models perform better than all three baselines across the datasets considered. We observe that there is no significant difference in test accuracy between the inductive and transductive settings. Among the different combine methods, the drop feature seems to perform best in general; however, the difference in performance between these variations is small. Also, Attn-MBGNN is slightly better, owing to the model's higher expressivity due to learnable attention weights.
Figure \ref{fig:train_plot} compares the training and test accuracy of all the models, with the standard supervised baseline versus the number of epochs on the CIFAR-10 dataset. The MBGNN models train faster than the standard network, giving a significant performance difference during early parts of the training process.

\subsection{Ablation study: $\mathbf{k}$ and batch size}
The most important parameters for graph-based learning are the size of the graph and the degree of each node, i.e., the neighborhood size $k$. Here we use the ResNet-50 encoder and the CIFAR-10 dataset for all our experiments. We expect the best value for $k$ to be close to the batch size divided by the number of classes, since this is the expected number of samples in a mini-batch having the same class label. For CIFAR-10 with a batch size of $256$, this value is $\sfrac{256}{10} \approx 25$. We also expect performance to improve with larger batch sizes, since this translates to larger graphs, and hence more samples per class.
We restrict ourselves to inductive predictions and provide the results for the drop feature model in Tables \ref{table:ablation_k} and \ref{table:ablation_BS}. 
The results match our prediction. First, for CIFAR-10 the best performance is indeed when $k$ is close to 25. Second, classification accuracy improves with larger batch sizes.

\begin{table*}[hbt!]
\begin{center}
\begin{tabular}{|l|c|c|c|c|c|c|c|}
\hline
\textbf{Model} & \textbf{k=4} & \textbf{k=8} & \textbf{k=16} &\textbf{k=32} & \textbf{k=64} & \textbf{k=128} & \textbf{k=256} \\ \hline
MBGNN (dropfeat)            & 94.91 & 95.10 & 95.24 & 95.22 & 95.02 & 94.11 & 92.88
\\
Attn-MBGNN (dropfeat)       &  94.79 & 94.88 & 95.05 & 94.97 & 94.90 & 94.43 & 93.52
\\
\hline
\end{tabular}
\caption{Image classification results on CIFAR-10 using a Resnet-50 encoder and MBGNN, while varying the neighbourhood size $k$. All the networks are trained with a batch size of 256 using a single layer GNN.}
\label{table:ablation_k}
\end{center}
\end{table*}

\begin{table*}[hbt!]
\begin{center}
\begin{tabular}{|l|c|c|c|c|c|}
\hline
\textbf{Model} & \textbf{BS=32} & \textbf{BS=64} & \textbf{BS=128} & \textbf{BS=256} & \textbf{BS=512} \\ \hline
MBGNN (dropfeat)           & 94.34 & 94.88 & 95.10 & 95.24 & 95.28
\\
Attn-MBGNN (dropfeat)      & 94.38 & 94.82 & 94.95 & 95.05 & 95.12
\\ \hline
\end{tabular}
\caption{Image classification results on CIFAR-10 using a Resnet-50 encoder and an MBGNN, with different batch sizes. All the networks are trained with a neighbourhood size of $k$ = 16, using a single layer GNN.}
\label{table:ablation_BS}
\end{center}
\end{table*}

\begin{figure*}[hbt!]
\begin{center}
    \begin{subfigure}[b]{0.28\textwidth}
        \centering
        \includegraphics[width=\textwidth]{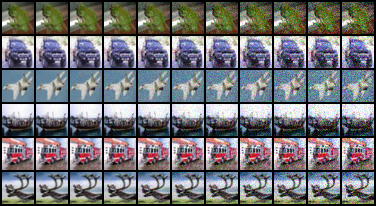}
        \vspace{0.14cm}
        \caption{}
    \end{subfigure}
    \begin{subfigure}[b]{0.28\textwidth}
        \centering
        \includegraphics[width=\textwidth]{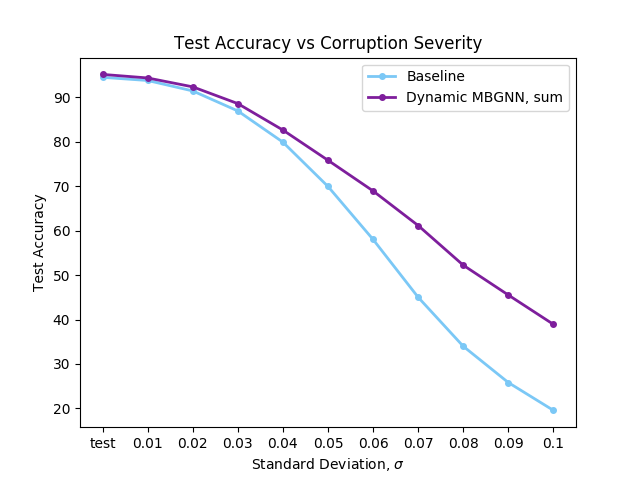}
        \caption{}
    \end{subfigure}
    \begin{subfigure}[b]{0.28\textwidth}
        \centering
        \includegraphics[width=\textwidth]{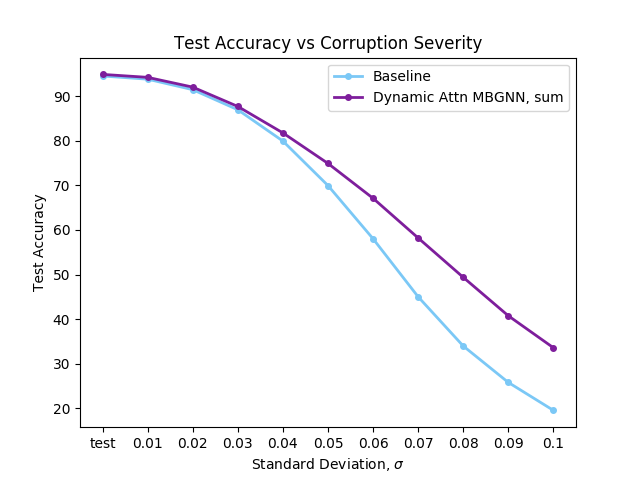}
        \caption{}
    \end{subfigure}
    \begin{subfigure}[b]{0.28\textwidth}
        \centering
        \includegraphics[width=\textwidth]{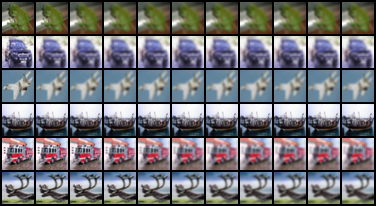}
        \vspace{0.14cm}
        \caption{}
    \end{subfigure}
    \begin{subfigure}[b]{0.28\textwidth}
        \centering
        \includegraphics[width=\textwidth]{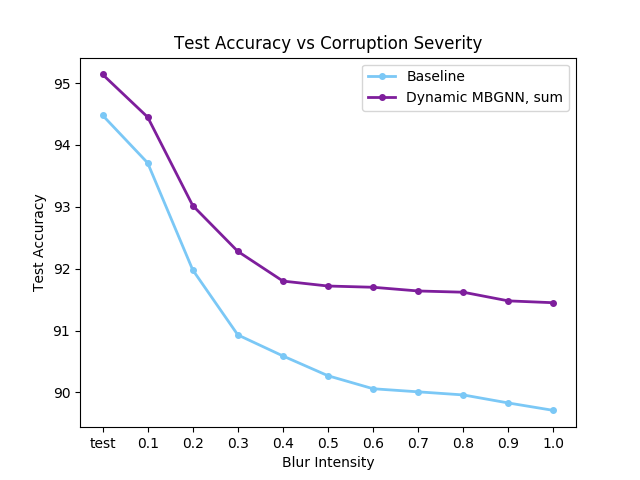}
        \caption{}
    \end{subfigure}
    \begin{subfigure}[b]{0.28\textwidth}
        \centering
        \includegraphics[width=\textwidth]{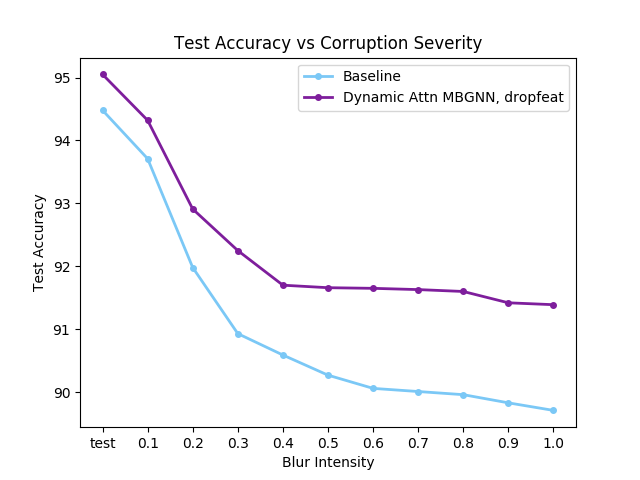}
        \caption{}
    \end{subfigure} \\
    \caption{We encourage the reader to zoom-in on the PDF. Average test accuracy at different corruption severities for pixel-wise Gaussian noise (top) and Gaussian blurring (bottom) on CIFAR10, using a ResNet-50 encoder. Models using MBGNNs (purple) maintain higher accuracy over the range of corruption severities when compared to the baseline model (blue) and have a lower reduction in accuracy for higher corruption levels. (a) Sample images with increasing level of Gaussian noise, (b) and (c) test accuracy plots for Attn-MBGNN (sum) and Attn MBGNN (sum), (d) Sample images with increasing level of Gaussian blurring, (e) and (f) test accuracy plots for MBGNN (sum) and Attn MBGNN (sum). The plots for other models are provided in the supplementary material.}
    \label{fig:blur_plots}
\end{center}
\end{figure*}

\subsection{Robustness experiments}
\label{sec:robustness}

\begin{table*}[t]
\begin{center}
\begin{tabular}{|l|c|c|c|c|c|c|}
\hline
 & \multicolumn{3}{c|}{\bf SimBA} & \multicolumn{3}{c|}{\bf Bandits-TD} \\ \cline{2-7}
 \bf Model & \bf \makecell{Mean \\ Queries} & \bf \makecell{Median \\ Queries} & \bf \makecell{Attack \\ Success rate} & \bf \makecell{Mean \\ Queries} & \bf \makecell{Median \\ Queries} & \bf \makecell{Attack \\ Success rate} \\
\hline
Baseline ResNet-50 & 357.31 & 302 & 100.00 \% & 564.22 & 524 & 87.93 \% \\ \hline
MBGNN, concat & 508.53 & 388 & 99.79 \% & 768.86 & 620 & 87.77 \% \\
MBGNN, sum & 520.46 & \bf 394 & 99.89 \% &  708.55 & 638 & 87.74 \% \\
MBGNN, dropfeat & \bf 572.15 & 387 & \bf 99.68 \% & \bf 787.69 & \bf 659 & \bf 87.66 \% \\ \hline
Attn MBGNN, concat & \bf 415.11 & \bf 342 & 99.89 \% & \bf 613.15 & 574 & 87.92 \% \\
Attn MBGNN, sum & 358.88 & 319 & 100.00 \% & 558.31 & 542 & 87.94 \% \\
Attn MBGNN, dropfeat & 392.19 & 338 & \bf 99.79 \% & 590.17 & \bf 580 & \bf 87.84 \% \\
\hline
\end{tabular}
\end{center}
\caption{Black-box adversarial attack results for the baseline and the proposed models using simBA \cite{Guo2019SimpleBA} and Bandits-TD \cite{ilyas2018prior} on the CIFAR10 dataset. A higher number of queries is better, and a lower attack success rate is better. Note that the mean and median number of queries is calculated over successful attacks only.}
\label{table:attacks}
\end{table*}

\begin{figure*}[hbt!]
\begin{center}
    \begin{subfigure}[b]{0.45\textwidth}
        \centering
        \includegraphics[width=\textwidth]{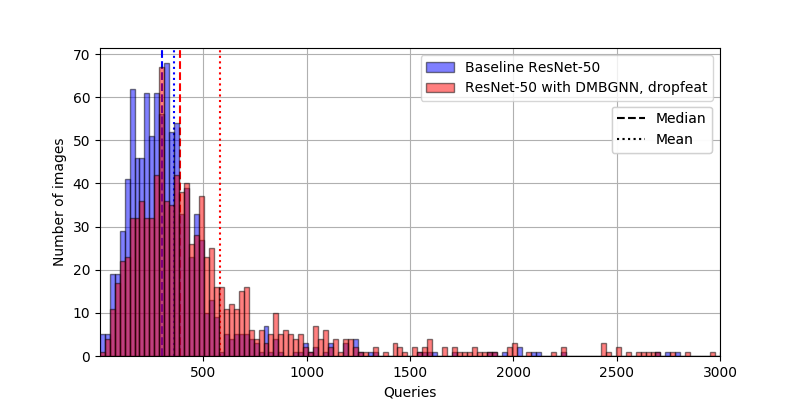}
        \caption{}
    \end{subfigure}
    \begin{subfigure}[b]{0.45\textwidth}
        \centering
        \includegraphics[width=\textwidth]{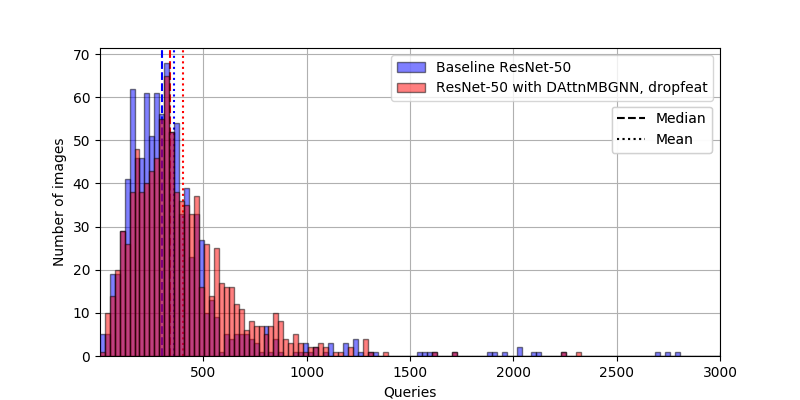}
        \caption{}
    \end{subfigure} \\
    \begin{subfigure}[b]{0.45\textwidth}
        \centering
        \includegraphics[width=\textwidth]{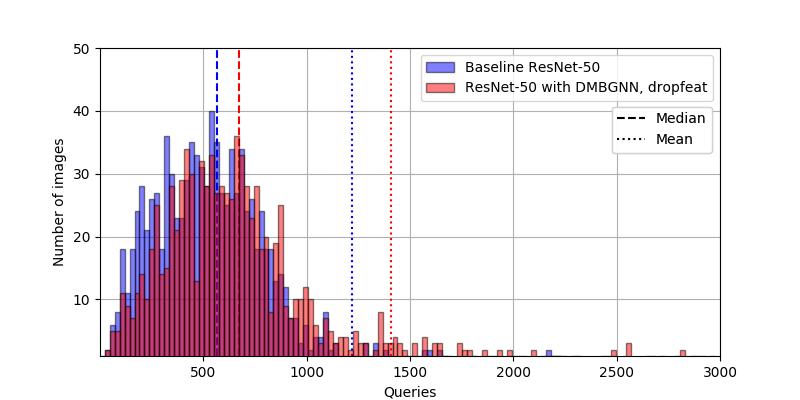}
        \caption{}
    \end{subfigure}
    \begin{subfigure}[b]{0.45\textwidth}
        \centering
        \includegraphics[width=\textwidth]{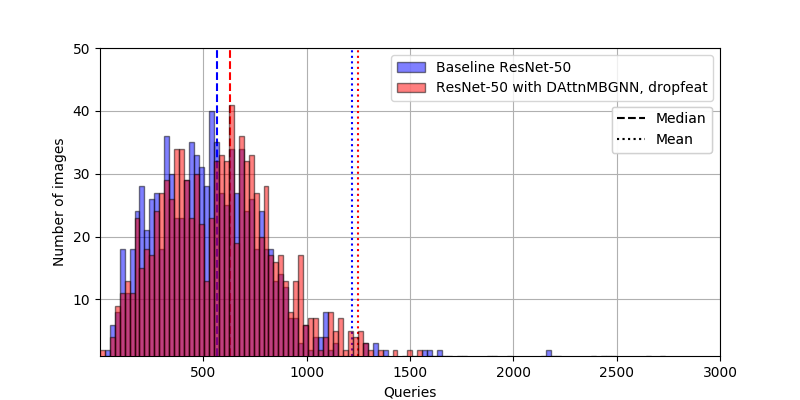}
        \caption{}
    \end{subfigure}
    \caption{Histograms of the number of queries required until a successful attack (over 1000 target images) on the CIFAR10 dataset using dropfeat MBGNN and Attn-MBGNN for simBA (top) and Bandits-TD (bottom). The queries axis is limited to 3000 queries for clarity of presentation. Models using MBGNNs (red) require more queries on average for a successful attack as compared to the baseline model (blue). The plots for other model variations are provided in the supplementary material.}
    \label{fig:attack_query_plots}
\end{center}
\end{figure*}

In order to test the robustness of MBGNNs to image perturbations, we first consider random (pixel-wise) Gaussian noise and local Gaussian blurring on input images, with varying levels of corruption severity. The evaluation of test accuracy is done via transductive testing, where a mini-batch consists of a single corrupted image along with a training set of uncorrupted images. The class label prediction made by the model for the corrupted image is compared against the true label. Figure \ref{fig:blur_plots} shows plots of test accuracy, measured in the manner described above, for the best performing variations of MBGNN (sum) and Attn-MBGNN (sum) on the CIFAR-10 dataset for different levels of corruption severity. We observe that models using MBGNNs are far better at accommodating the effects of local perturbations to the images (Fig. \ref{fig:blur_plots} first row) and that although the effects of local Gaussian blur are less harmful, models based on MBGNNs are still better by about 2\% over the range of corruption severities we have considered (Fig. \ref{fig:blur_plots} second row).


We also test the robustness of the model to two recently proposed and popular black-box adversarial attack methods, simBA \cite{Guo2019SimpleBA}, and Bandits-TD \cite{ilyas2018prior}. Table \ref{table:attacks} shows the mean and median number of queries before a successful attack, and the attack success rate, for different MBGNN models for the two attack methods. The MBGNN models have slightly lower attack success rates and higher mean/median queries before a successful attack when compared to the baseline ResNet model. The large increase in mean queries can be attributed to a heavier tail in the distribution of queries, as can be seen from the histogram plots in Figure \ref{fig:attack_query_plots}, which show results for the best performing variations of MBGNN (dropfeat) and Attn-MBGNN (dropfeat). Plots for other variations are provided in the supplementary material. The Attention MBGNN models outperform the baseline model but do not perform as well as the MBGNN models. One reason for this might be that any perturbation in the input samples has a compounding effect on calculating the attention weights and, therefore, the aggregation.

\subsection{GAN training using MC-MBGNN}
\label{sec:modecollapse}

\begin{figure*}[h!]
\begin{center}
    \begin{subfigure}[b]{0.4\textwidth}
        \centering
        \includegraphics[width=\textwidth]{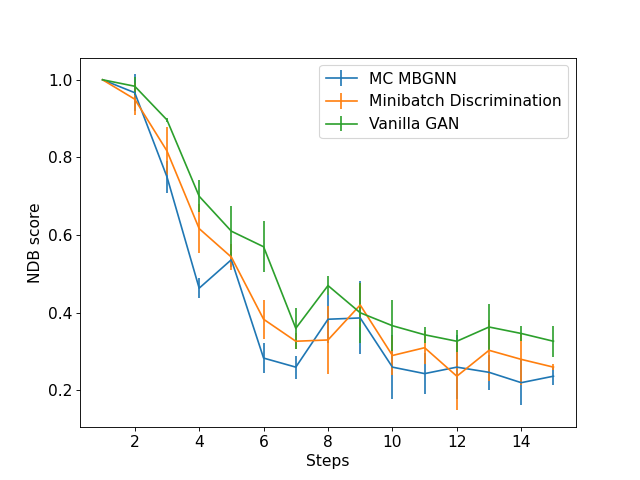}
        \caption{CelebA}
    \end{subfigure}
    \begin{subfigure}[b]{0.4\textwidth}
        \centering
        \includegraphics[width=\textwidth]{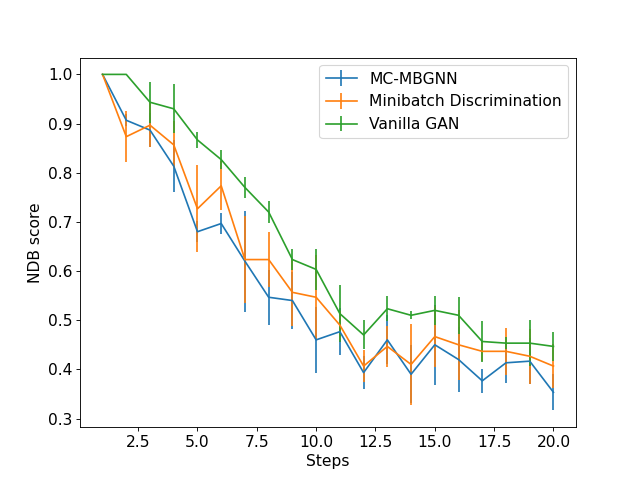}
        \caption{CIFAR 10}
    \end{subfigure}
    \caption{Plots of NDB scores for 100 bins over the steps, where a step refers to 500 training iterations. Lower NDB scores imply higher sample diversity. We provide the confidence of the NDB values over 5 runs. We take 160 batches of real training images and 40 batches of fake generated images with a batch size of 128 to estimate the true statistics of the dataset and reduce the NDB test time.}
    \label{fig:GAN_NDB}
\end{center}
\end{figure*}

Lastly, we provide results for training GANs using MC-MBGNN, and compare this strategy with both a vanilla GAN and minibatch discrimination \cite{DBLP:journals/corr/SalimansGZCRC16}. We test the diversity of our generated samples using the technique of Number of statistically Different Bins (NDB), which was proposed in \cite{richardson2018gans} as a metric to quantify mode collapse in GANs. To compute this metric, we first cluster the training dataset in $K$ different bins and then allocate the generated images in these bins based on their proximity to the centroid of each bin. Then, we measure the statistical similarity between the real and fake images in each of the bins and compute the fraction of statistically different bins that give us the NDB score. In the case of mode collapse, the number of statically different bins is close to $K$, and the NDB score is close to $1$. Using the NDB score as the metric, in Figure \ref{fig:GAN_NDB}, we show that our proposed MC-MBGNN helps the generator learn faster and generate more diverse samples, which is substantiated by lower NDB scores. We provide the details of the architecture, experiment setup and generated images in the Supplementary material.

\section{Discussion}
The use of an MBGNN shows promise for computer vision tasks, and in particular, image classification. It improves raw performance and adds robustness to adversarial attacks and image perturbations, with low computational overhead. There are certain practical limitations of our method in its present form for training on large-scale datasets, which could be addressed in future work. First, from an implementation standpoint, parallelization requires further development. In our experiments, each mini-batch was handled by the same GPU. Second, the mini-batch size also needs to be at least a few times larger than the number of classes to ensure that multiple samples belong to the same class within a mini-batch during training.
MBGNNs can be used with the most popular network models and require a modification of only their last layer. As such, we anticipate that they could find a use for diverse computer vision tasks beyond classification, such as segmentation, region proposal detection, image captioning, and relationship modeling.


\bibliographystyle{ieee_fullname}
\bibliography{references}

\begin{thebibliography}{10}\itemsep=-1pt

\bibitem{agustsson2019generative}
Eirikur Agustsson, Michael Tschannen, Fabian Mentzer, Radu Timofte, and Luc~Van
  Gool.
\newblock Generative adversarial networks for extreme learned image
  compression.
\newblock In {\em Proceedings of the IEEE International Conference on Computer
  Vision}, pages 221--231, 2019.

\bibitem{Benenson_2019_CVPR}
Rodrigo Benenson, Stefan Popov, and Vittorio Ferrari.
\newblock Large-scale interactive object segmentation with human annotators.
\newblock In {\em Proceedings of the IEEE/CVF Conference on Computer Vision and
  Pattern Recognition (CVPR)}, June 2019.

\bibitem{bruna2013spectral}
Joan Bruna, Wojciech Zaremba, Arthur Szlam, and Yann LeCun.
\newblock Spectral networks and locally connected networks on graphs.
\newblock {\em arXiv preprint arXiv:1312.6203}, 2013.

\bibitem{chen2020simple}
Ting Chen, Simon Kornblith, Mohammad Norouzi, and Geoffrey Hinton.
\newblock A simple framework for contrastive learning of visual
  representations.
\newblock {\em arXiv preprint arXiv:2002.05709}, 2020.

\bibitem{defferrard2016convolutional}
Micha{\"e}l Defferrard, Xavier Bresson, and Pierre Vandergheynst.
\newblock Convolutional neural networks on graphs with fast localized spectral
  filtering.
\newblock In {\em Advances in neural information processing systems}, pages
  3844--3852, 2016.

\bibitem{gong2019exploiting}
Liyu Gong and Qiang Cheng.
\newblock Exploiting edge features for graph neural networks.
\newblock In {\em Proceedings of the IEEE Conference on Computer Vision and
  Pattern Recognition}, pages 9211--9219, 2019.

\bibitem{goodfellow2014generative}
Ian Goodfellow, Jean Pouget-Abadie, Mehdi Mirza, Bing Xu, David Warde-Farley,
  Sherjil Ozair, Aaron Courville, and Yoshua Bengio.
\newblock Generative adversarial nets.
\newblock In {\em Advances in neural information processing systems}, pages
  2672--2680, 2014.

\bibitem{gori2005new}
Marco Gori, Gabriele Monfardini, and Franco Scarselli.
\newblock A new model for learning in graph domains.
\newblock In {\em Proceedings. 2005 IEEE International Joint Conference on
  Neural Networks, 2005.}, volume~2, pages 729--734. IEEE, 2005.

\bibitem{Guo2019SimpleBA}
Chuan Guo, Jacob~R. Gardner, Yurong You, A. Wilson, and Kilian~Q. Weinberger.
\newblock Simple black-box adversarial attacks.
\newblock In {\em ICML}, 2019.

\bibitem{he2016deep}
Kaiming He, Xiangyu Zhang, Shaoqing Ren, and Jian Sun.
\newblock Deep residual learning for image recognition.
\newblock In {\em Proceedings of the IEEE conference on computer vision and
  pattern recognition}, pages 770--778, 2016.

\bibitem{henaff2015deep}
Mikael Henaff, Joan Bruna, and Yann LeCun.
\newblock Deep convolutional networks on graph-structured data.
\newblock {\em arXiv preprint arXiv:1506.05163}, 2015.

\bibitem{ilyas2018prior}
Andrew Ilyas, Logan Engstrom, and Aleksander Madry.
\newblock Prior convictions: Black-box adversarial attacks with bandits and
  priors.
\newblock In {\em International Conference on Learning Representations}, 2019.

\bibitem{Isaacs_2020_CVPR}
Or Isaacs, Oran Shayer, and Michael Lindenbaum.
\newblock Enhancing generic segmentation with learned region representations.
\newblock In {\em Proceedings of the IEEE/CVF Conference on Computer Vision and
  Pattern Recognition (CVPR)}, June 2020.

\bibitem{isola2017image}
Phillip Isola, Jun-Yan Zhu, Tinghui Zhou, and Alexei~A Efros.
\newblock Image-to-image translation with conditional adversarial networks.
\newblock In {\em Proceedings of the IEEE conference on computer vision and
  pattern recognition}, pages 1125--1134, 2017.

\bibitem{khosla2020supervised}
Prannay Khosla, Piotr Teterwak, Chen Wang, Aaron Sarna, Yonglong Tian, Phillip
  Isola, Aaron Maschinot, Ce Liu, and Dilip Krishnan.
\newblock Supervised contrastive learning.
\newblock {\em arXiv preprint arXiv:2004.11362}, 2020.

\bibitem{kipf2016semi}
Thomas~N Kipf and Max Welling.
\newblock Semi-supervised classification with graph convolutional networks.
\newblock {\em arXiv preprint arXiv:1609.02907}, 2016.

\bibitem{krizhevsky2012}
Alex Krizhevsky, Ilya Sutskever, and Geoffrey~E. Hinton.
\newblock Imagenet classification with deep convolutional neural networks.
\newblock In {\em Proceedings of the 25th International Conference on Neural
  Information Processing Systems - Volume 1}, NIPS'12, page 1097–1105, Red
  Hook, NY, USA, 2012. Curran Associates Inc.

\bibitem{ledig2017photo}
Christian Ledig, Lucas Theis, Ferenc Husz{\'a}r, Jose Caballero, Andrew
  Cunningham, Alejandro Acosta, Andrew Aitken, Alykhan Tejani, Johannes Totz,
  Zehan Wang, et~al.
\newblock Photo-realistic single image super-resolution using a generative
  adversarial network.
\newblock In {\em Proceedings of the IEEE conference on computer vision and
  pattern recognition}, pages 4681--4690, 2017.

\bibitem{pmlr-v48-niepert16}
Mathias Niepert, Mohamed Ahmed, and Konstantin Kutzkov.
\newblock Learning convolutional neural networks for graphs.
\newblock volume~48 of {\em Proceedings of Machine Learning Research}, pages
  2014--2023, New York, New York, USA, 20--22 Jun 2016. PMLR.

\bibitem{ren2015faster}
Shaoqing Ren, Kaiming He, Ross Girshick, and Jian Sun.
\newblock Faster r-cnn: Towards real-time object detection with region proposal
  networks.
\newblock In {\em Advances in neural information processing systems}, pages
  91--99, 2015.

\bibitem{richardson2018gans}
Eitan Richardson and Yair Weiss.
\newblock On gans and gmms.
\newblock In {\em Advances in Neural Information Processing Systems}, pages
  5847--5858, 2018.

\bibitem{rodriguez2020embedding}
Pau Rodríguez, Issam Laradji, Alexandre Drouin, and Alexandre Lacoste.
\newblock Embedding propagation: Smoother manifold for few-shot classification.
\newblock {\em arXiv preprint arXiv:2003.04151}, 2020.

\bibitem{DBLP:journals/corr/SalimansGZCRC16}
Tim Salimans, Ian~J. Goodfellow, Wojciech Zaremba, Vicki Cheung, Alec Radford,
  and Xi Chen.
\newblock Improved techniques for training gans.
\newblock {\em CoRR}, abs/1606.03498, 2016.

\bibitem{satorras2018few}
Victor~Garcia Satorras and Joan~Bruna Estrach.
\newblock Few-shot learning with graph neural networks.
\newblock In {\em International Conference on Learning Representations}, 2018.

\bibitem{scarselli2008graph}
Franco Scarselli, Marco Gori, Ah~Chung Tsoi, Markus Hagenbuchner, and Gabriele
  Monfardini.
\newblock The graph neural network model.
\newblock {\em IEEE Transactions on Neural Networks}, 20(1):61--80, 2008.

\bibitem{velickovic2018graph}
Petar Veličković, Guillem Cucurull, Arantxa Casanova, Adriana Romero, Pietro
  Liò, and Yoshua Bengio.
\newblock Graph attention networks.
\newblock In {\em International Conference on Learning Representations}, 2018.

\bibitem{wang2020affinity}
Chu Wang, Babak Samari, Vladimir~G Kim, Siddhartha Chaudhuri, and Kaleem
  Siddiqi.
\newblock Affinity graph supervision for visual recognition.
\newblock In {\em Proceedings of the IEEE/CVF Conference on Computer Vision and
  Pattern Recognition}, pages 8247--8255, 2020.

\bibitem{wang2018high}
Ting-Chun Wang, Ming-Yu Liu, Jun-Yan Zhu, Andrew Tao, Jan Kautz, and Bryan
  Catanzaro.
\newblock High-resolution image synthesis and semantic manipulation with
  conditional gans.
\newblock In {\em Proceedings of the IEEE conference on computer vision and
  pattern recognition}, pages 8798--8807, 2018.

\bibitem{wang2019dynamic}
Yue Wang, Yongbin Sun, Ziwei Liu, Sanjay~E Sarma, Michael~M Bronstein, and
  Justin~M Solomon.
\newblock Dynamic graph cnn for learning on point clouds.
\newblock {\em Acm Transactions On Graphics (tog)}, 38(5):1--12, 2019.

\bibitem{xu2018representation}
Keyulu Xu, Chengtao Li, Yonglong Tian, Tomohiro Sonobe, Ken-ichi Kawarabayashi,
  and Stefanie Jegelka.
\newblock Representation learning on graphs with jumping knowledge networks.
\newblock {\em arXiv preprint arXiv:1806.03536}, 2018.

\bibitem{zhang2017relationship}
Ji Zhang, Mohamed Elhoseiny, Scott Cohen, Walter Chang, and Ahmed Elgammal.
\newblock Relationship proposal networks.
\newblock In {\em Proceedings of the IEEE Conference on Computer Vision and
  Pattern Recognition}, pages 5678--5686, 2017.

\bibitem{places2017tpami}
B. {Zhou}, A. {Lapedriza}, A. {Khosla}, A. {Oliva}, and A. {Torralba}.
\newblock Places: A 10 million image database for scene recognition.
\newblock {\em IEEE Transactions on Pattern Analysis and Machine Intelligence},
  40(6):1452--1464, 2018.

\bibitem{zhu2020generalizing}
Jiong Zhu, Yujun Yan, Lingxiao Zhao, Mark Heimann, Leman Akoglu, and Danai
  Koutra.
\newblock Generalizing graph neural networks beyond homophily.
\newblock {\em arXiv preprint arXiv:2006.11468}, 2020.

\bibitem{zhu2017unpaired}
Jun-Yan Zhu, Taesung Park, Phillip Isola, and Alexei~A Efros.
\newblock Unpaired image-to-image translation using cycle-consistent
  adversarial networks.
\newblock In {\em Proceedings of the IEEE international conference on computer
  vision}, pages 2223--2232, 2017.

\bibitem{zhuang2019local}
Chengxu Zhuang, Xuehao Ding, Divyanshu Murli, and Daniel Yamins.
\newblock Local label propagation for large-scale semi-supervised learning,
  2019.

\end{thebibliography}
\clearpage
\section{Supplementary Material}
\subsection{Results for WideResnet}
Table \ref{table:supervised_wideresnet} provides results for our MBGNN with all the combine options, using Wide ResNet-28-10 as the encoder network. We also use this network for all the baselines in this table. We observe a consistent improvement across all datasets, which is consistent with the results of our experiments using ResNet-50 in the main paper.


\begin{table*}[h]
\small
\begin{tabular}{|l|c|c|c|c|c|c|}
\hline
 & \multicolumn{2}{c|}{\bf CIFAR 10} & \multicolumn{2}{c|}{\bf CIFAR 100} & \multicolumn{2}{c|}{\bf MIT 67} \\ \cline{2-7}
 \bf Model & \bf Inductive     & \bf Transductive  & \bf Inductive    & \bf Transductive    & \bf Inductive   & \bf Transductive  \\ \hline
Supervised vanilla           & \multicolumn{2}{c|}{$95.62 \scriptstyle{\pm 0.14}$} & \multicolumn{2}{c|}{$79.58 \scriptstyle{\pm 0.20}$} &  \multicolumn{2}{c|}{$66.20 \scriptstyle{\pm 0.19}$}             
\\
Supervised contrastive \cite{khosla2020supervised}      & \multicolumn{2}{c|}{$95.91 \scriptstyle{\pm 0.16}$}  &\multicolumn{2}{c|}{$80.15 \scriptstyle{\pm 0.15}$} &         \multicolumn{2}{c|}{$66.89 \scriptstyle{\pm 0.17}$}        
\\
Affinity supervision \cite{wang2020affinity}      & \multicolumn{2}{c|}{$95.59 \scriptstyle{\pm 0.18}$} &\multicolumn{2}{c|}{$79.8 \scriptstyle{\pm 0.19}$} & \multicolumn{2}{c|}{$66.8 \scriptstyle{\pm 0.16}$}        
\\ \hline
MBGNN (concat)            &  $96.02 \scriptstyle{\pm 0.21}$  &   $96.05 \scriptstyle{\pm 0.20}$   &  $80.18 \scriptstyle{\pm 0.18}$   &      $80.20 \scriptstyle{\pm 0.18}$  &  $66.87 \scriptstyle{\pm 0.22}$       &  $66.90 \scriptstyle{\pm 0.21}$  \\
MBGNN (sum)               & $95.78 \scriptstyle{\pm 0.15}$   &   $95.80 \scriptstyle{\pm 0.14}$   & $79.85 \scriptstyle{\pm 0.17}$   &    $79.88 \scriptstyle{\pm 0.15}$   & $66.52 \scriptstyle{\pm 0.18}$       &  $66.51 \scriptstyle{\pm 0.17}$    \\
MBGNN (dropfeat)          & $\bf 96.14 \scriptstyle{\pm 0.16}$   &   $\bf 96.17 \scriptstyle{\pm 0.18}$  & $80.46 \scriptstyle{\pm 0.21}$  &     $80.45 \scriptstyle{\pm 0.20}$  & $67.10 \scriptstyle{\pm 0.19} $      &  $67.12 \scriptstyle{\pm 0.20}$
\\ \hline
Attn-MBGNN (concat)       & $95.95 \scriptstyle{\pm 0.15}$   &   $95.95 \scriptstyle{\pm 0.18}$  & $80.22 \scriptstyle{\pm 0.20}$  &    $80.23 \scriptstyle{\pm 0.19}$ &  $66.96 \scriptstyle{\pm 0.18}$       & $66.98 \scriptstyle{\pm 0.18}$
\\
Attn-MBGNN (sum)          & $95.89 \scriptstyle{\pm 0.19}$   &   $95.91 \scriptstyle{\pm 0.18}$   & $80.02 \scriptstyle{\pm 0.20}$  &    $ 80.06 \scriptstyle{\pm 0.18}$ & $66.67 \scriptstyle{\pm 0.16}$       &  $66.69 \scriptstyle{\pm 0.17}$
\\
Attn-MBGNN (dropfeat)     & $96.12 \scriptstyle{\pm 0.20}$    &  $96.14 \scriptstyle{\pm 0.21}$  & $\bf 80.76 \scriptstyle{\pm 0.17}$  &    $\bf 80.76 \scriptstyle{\pm 0.15}$ &   $\bf 67.25 \scriptstyle{\pm 0.22}$       &  $\bf 67.26 \scriptstyle{\pm 0.20} $       
\\ \hline
\end{tabular}
\caption{Image classification results using a Wide ResNet-28-10 encoder. The architectures are trained using a batch size of 256 and with $k$ = 16 for CIFAR-10, and $k$ = 4 for CIFAR-100 and MIT67. We provide results for different combine modes of our single layered mini-batch graph based models (rows 4-9).}
\label{table:supervised_wideresnet}
\end{table*}

\subsection{Additional robustness results}

\subsubsection{Random Perturbations}
We provide results for Gaussian noise and Gaussian blurring perturbation for all the different variations of the MBGNN models, using the ResNet-50 encoder on the CIFAR-10 dataset. Figure \ref{fig:samples_noisy} and \ref{fig:samples_blur} show some sample CIFAR-10 images, with different levels of corruption severity for visualization. For both Gaussian noise and Gaussian blurring we define corruption severity as the standard deviation, $\sigma$ used when sampling from the Gaussian distribution, with higher values of $\sigma$ corresponding to increased corruption.

Figure \ref{fig:noise_plots} shows the average test accuracy for different levels of Gaussian noise and Figure \ref{fig:blur_plots_2} shows the average test accuracy for different levels of Gaussian blurring. In both figures, the top row shows plots for MBGNN models (concat, sum and dropfeat) and the bottom row shows plots for Attn-MBGNN models (concat, sum and dropfeat). Models using MBGNNs (purple) maintain higher accuracy over the entire range of corruption severities as compared to the baseline model (blue) and show a lower drop in accuracy for higher corruption levels. We also observe that the sum combination method generally performs better than the other variations.

\subsubsection{Black-box Adversarial Attacks}
We provide histogram plots of the number of queries required until a successful attack (over 1000 images) for all the variations of the MBGNN model. Figure \ref{fig:simba_query_plots} and \ref{fig:bandits_query_plots} show the plots for SimBA \cite{Guo2019SimpleBA} and Bandits-TD \cite{ilyas2018prior}, respectively. The dashed lines indicate the median value and the dotted lines indicate the mean value for the different models. In both figures, the top row shows plots for MBGNN models (concat, sum and dropfeat) and the bottom row shows plots for Attn-MBGNN models (concat, sum and dropfeat). The performance across the different combination methods is similar, although dropfeat models generally have a higher number of mean and median queries, due to the heavier tail in the distribution. Also, the MBGNN models with attention mechanism have a lower number of mean and median queries on average than the models without attention.

\clearpage
\subsection{Details of the GAN architecture}
We use standard generator and discriminator architectures for our GAN model. Let us denote the following operations,\\ 
Basic convolution: $Conv(in\_channels, out\_channels, filter\_size, stride, padding)$,\\ Linear layer: $Linear(in\_dim, out\_dim)$,\\ Deconvolution: $Conv\_trans(in\_channels, out\_channels, filter\_size, stride, padding, output\_padding)$\\
The generator architecture using the above notation, is given by\\
$Linear(128,16384) - batch\_norm - ReLU\\
Conv\_trans( 1024, 512, 5, 2, 2, 1) - batch\_norm - ReLU\\ 
Conv\_trans( 512, 256, 5, 2, 2, 1) - batch\_norm - ReLU\\
Conv\_trans( 256, 128, 5, 2, 2, 1) - batch\_norm - ReLU\\
Conv\_trans( 128, 64, 5, 2, 2, 1) - batch\_norm - ReLU\\
Conv\_trans( 64, 3, 5, 1, 2, 0) - tanh$\\
The discriminator architecture is given by\\
$Conv( 3, 64, 4, 2, 1) - batch\_norm - LeakyReLU\\
Conv( 64, 128, 4, 2, 1) - batch\_norm - LeakyReLU\\
Conv( 128, 256, 4, 2, 1) - batch\_norm - LeakyReLU\\
Conv( 256, 512, 4, 2, 1) - batch\_norm - LeakyReLU\\
Linear(8192,1) - sigmoid$\\
This defines the GAN architecture for the CelebA dataset. For CIFAR-10, we slightly modify the $Linear(\cdot)$ layers to adjust to the reduced image sizes. We add both the minibatch discrimination and MC-MBGNN layers before the final $Linear(\cdot)$ layer of the discriminator. For the sake of our experiments, we use the Adam optimizer with a learning rate of $10^{-4}$ for the discriminator and $2 \times 10^{-4}$ for the generator, with a batch size of $128$. We provide the generated images using different discriminator architectures in Figure \ref{fig:gen_images} and Figure \ref{fig:gen_images_2}. We train the model on CIFAR-10 for $10000$ iterations and on CelebA for $7500$ iterations.


\clearpage

\begin{figure*}[hbt!]
\begin{center}
    \centering
    \includegraphics[width=0.7\textwidth]{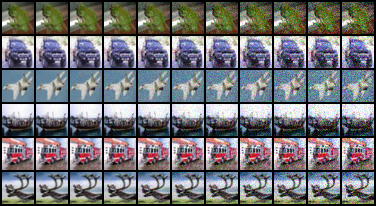}
    \caption{Sample images from the CIFAR-10 dataset with each column showing an increasing level of corruption severity, for pixelwise Gaussian noise.}
    \label{fig:samples_noisy}
\end{center}
\end{figure*}

\begin{figure*}[hbt!]
\begin{center}
    \begin{subfigure}[b]{0.33\textwidth}
        \centering
        \includegraphics[width=\textwidth]{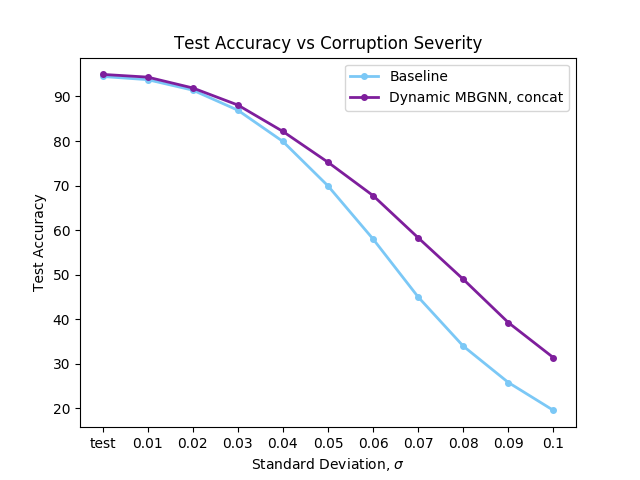}
        \caption{}
    \end{subfigure}
    \begin{subfigure}[b]{0.33\textwidth}
        \centering
        \includegraphics[width=\textwidth]{images/robustness/gnoise_dmbgnn_sum.png}
        \caption{}
    \end{subfigure}
    \begin{subfigure}[b]{0.33\textwidth}
        \centering
        \includegraphics[width=\textwidth]{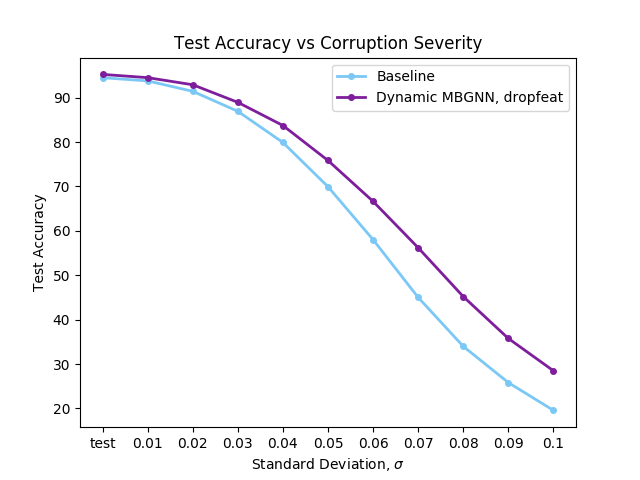}
        \caption{}
    \end{subfigure} \\
    \begin{subfigure}[b]{0.33\textwidth}
        \centering
        \includegraphics[width=\textwidth]{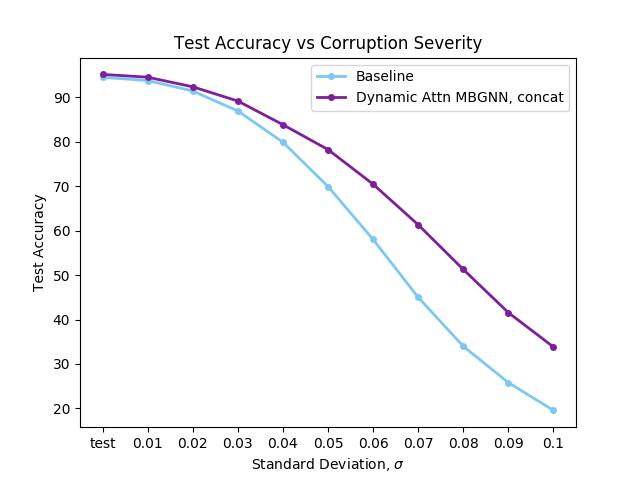}
        \caption{}
    \end{subfigure}
    \begin{subfigure}[b]{0.33\textwidth}
        \centering
        \includegraphics[width=\textwidth]{images/robustness/gnoise_dattnmbgnn_sum.png}
        \caption{}
    \end{subfigure}
    \begin{subfigure}[b]{0.33\textwidth}
        \centering
        \includegraphics[width=\textwidth]{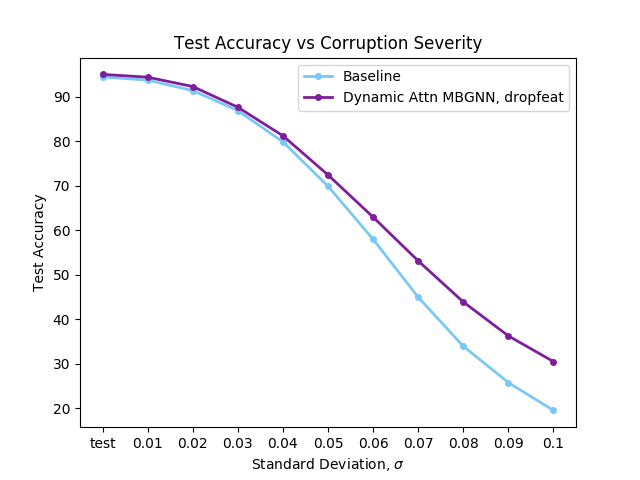}
        \caption{}
    \end{subfigure}
    \caption{Average test accuracy at different corruption severities for Gaussian noise on CIFAR10, using ResNet-50 with MBGNNs (top) and ResNet-50 with Attention MBGNNs (bottom). Models using MBGNNs (purple) maintain higher accuracy over the entire range of corruption severities as compared to the baseline model (blue), and show a lower drop in accuracy for higher corruption levels.}
    \label{fig:noise_plots}
\end{center}
\end{figure*}

\clearpage

\begin{figure*}[hbt!]
\begin{center}
    \centering
    \includegraphics[width=0.7\textwidth]{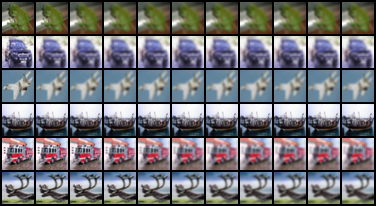}
    \caption{Sample images from CIFAR-10 dataset with each column showing increasing level of corruption severity for Gaussian blurring.}
    \label{fig:samples_blur}
\end{center}
\end{figure*}

\begin{figure*}[hbt!]
\begin{center}
    \begin{subfigure}[b]{0.33\textwidth}
        \centering
        \includegraphics[width=\textwidth]{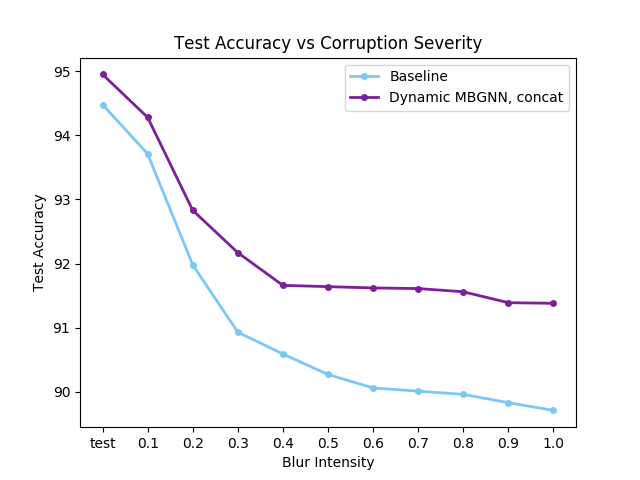}
        \caption{}
    \end{subfigure}
    \begin{subfigure}[b]{0.33\textwidth}
        \centering
        \includegraphics[width=\textwidth]{images/robustness/blur_dmbgnn_sum.png}
        \caption{}
    \end{subfigure}
    \begin{subfigure}[b]{0.33\textwidth}
        \centering
        \includegraphics[width=\textwidth]{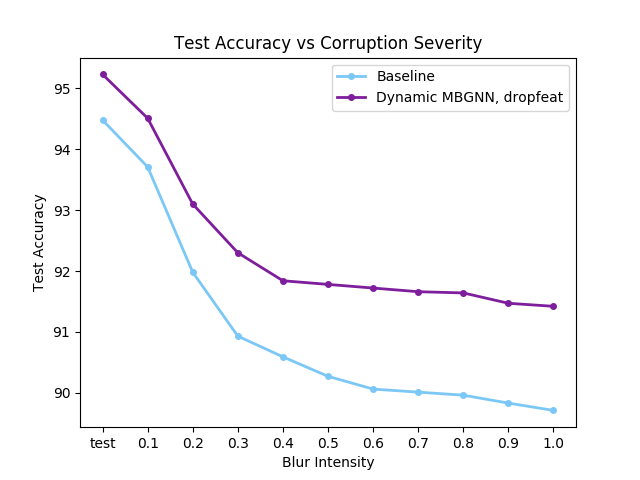}
        \caption{}
    \end{subfigure} \\
    \begin{subfigure}[b]{0.33\textwidth}
        \centering
        \includegraphics[width=\textwidth]{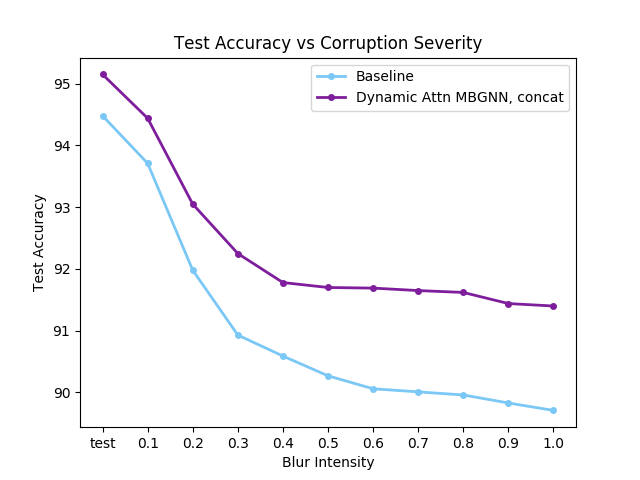}
        \caption{}
    \end{subfigure}
    \begin{subfigure}[b]{0.33\textwidth}
        \centering
        \includegraphics[width=\textwidth]{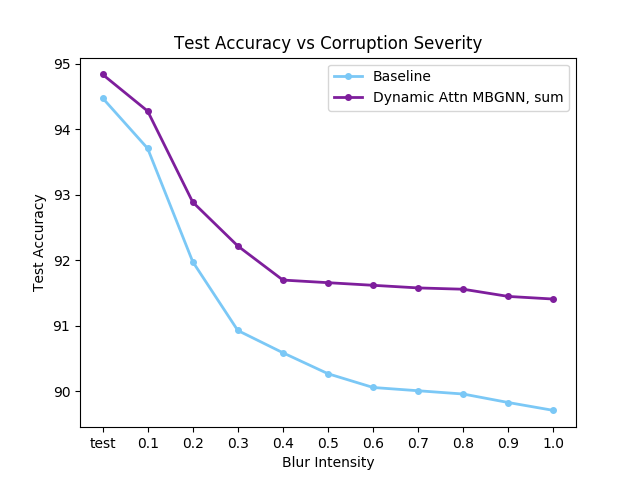}
        \caption{}
    \end{subfigure}
    \begin{subfigure}[b]{0.33\textwidth}
        \centering
        \includegraphics[width=\textwidth]{images/robustness/blur_dattnmbgnn_dropfeat.png}
        \caption{}
    \end{subfigure}
    \caption{Average test accuracy at different corruption severities for Gaussian blurring on CIFAR10, using ResNet-50 with MBGNNs (top) and ResNet-50 with Attention MBGNNs (bottom). Models using MBGNNs (purple) maintain higher accuracy over the range of corruption severities as compared to baseline model (blue) and have lower drop in accuracy for higher corruption levels.}
    \label{fig:blur_plots_2}
\end{center}
\end{figure*}

\clearpage

\begin{figure*}[hbt!]
\begin{center}
    \begin{subfigure}[b]{0.33\textwidth}
        \centering
        \includegraphics[width=\textwidth]{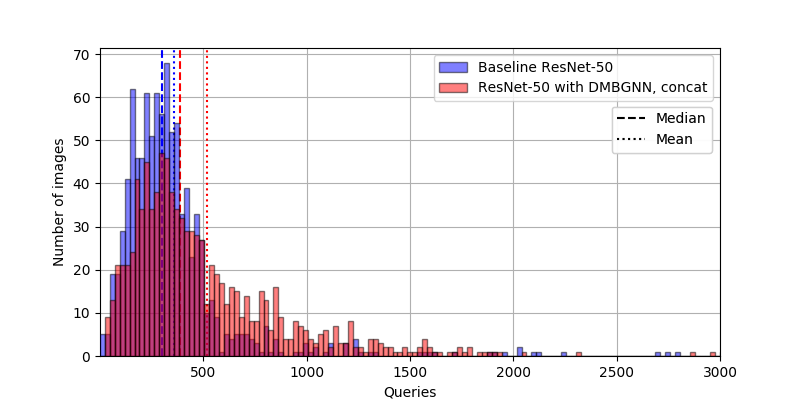}
        \caption{}
    \end{subfigure}
    \hfill
    \begin{subfigure}[b]{0.33\textwidth}
        \centering
        \includegraphics[width=\textwidth]{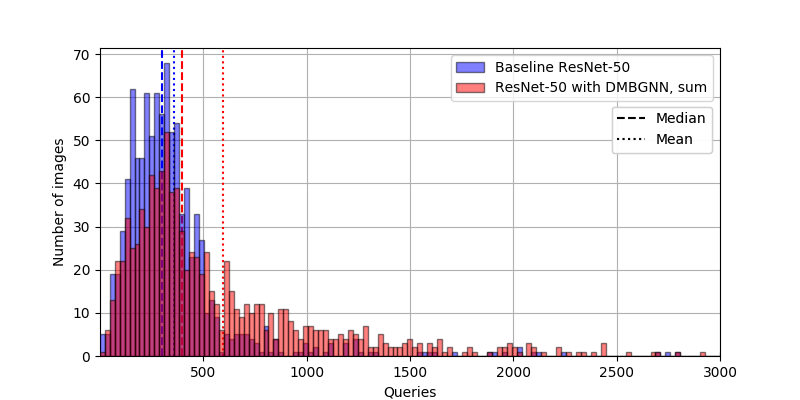}
        \caption{}
    \end{subfigure}
    \hfill
    \begin{subfigure}[b]{0.33\textwidth}
        \centering
        \includegraphics[width=\textwidth]{images/robustness/simba_dmbgnn_dropfeat.png}
        \caption{}
    \end{subfigure} \\
    \begin{subfigure}[b]{0.33\textwidth}
        \centering
        \includegraphics[width=\textwidth]{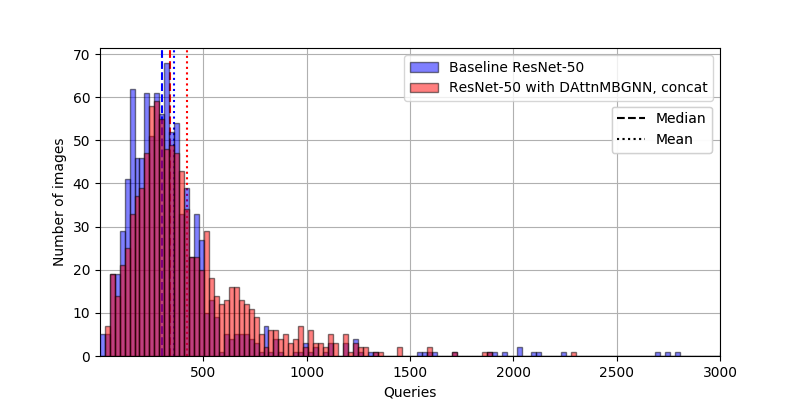}
        \caption{}
    \end{subfigure}
    \hfill
    \begin{subfigure}[b]{0.33\textwidth}
        \centering
        \includegraphics[width=\textwidth]{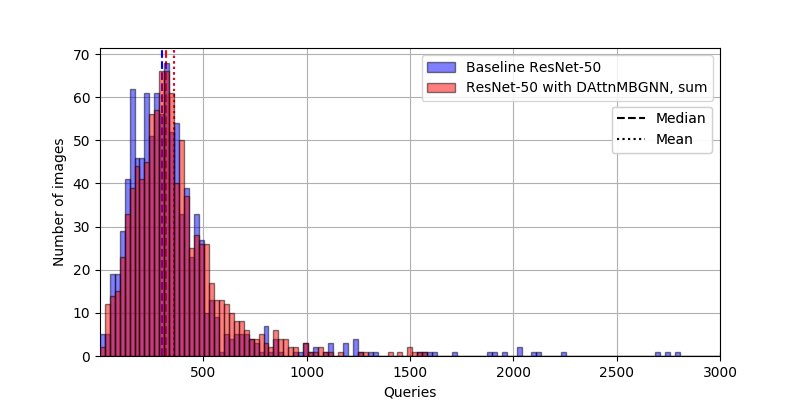}
        \caption{}
    \end{subfigure}
    \hfill
    \begin{subfigure}[b]{0.33\textwidth}
        \centering
        \includegraphics[width=\textwidth]{images/robustness/simba_dattnmbgnn_dropfeat.png}
        \caption{}
    \end{subfigure}
    \caption{Histogram of number of queries required until a successful attack (over 1000 target images) using simBA on the CIFAR10 dataset. The queries axis is limited to 3000 queries for clarity of presentation. Models using MBGNNs (red) require a larger number of queries on average for a successful, attack as compared to the baseline model (blue). (a) MBGNN, concat, (b) MBGNN, sum, (c) MBGNN, dropfeat, (d) Attn-MBGNN, concat, (e) Attn-MBGNN, sum, (f) Attn-MBGNN, dropfeat.}
    \label{fig:simba_query_plots}
\end{center}
\end{figure*}

\begin{figure*}[hbt!]
\begin{center}
    \begin{subfigure}[b]{0.33\textwidth}
        \centering
        \includegraphics[width=\textwidth]{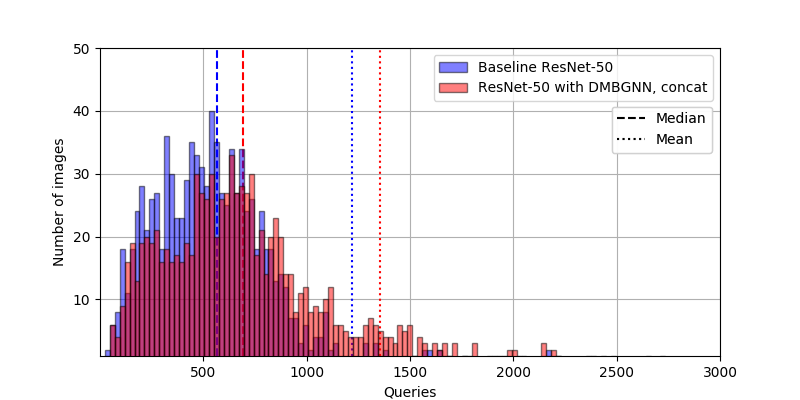}
        \caption{}
    \end{subfigure}
    \hfill
    \begin{subfigure}[b]{0.33\textwidth}
        \centering
        \includegraphics[width=\textwidth]{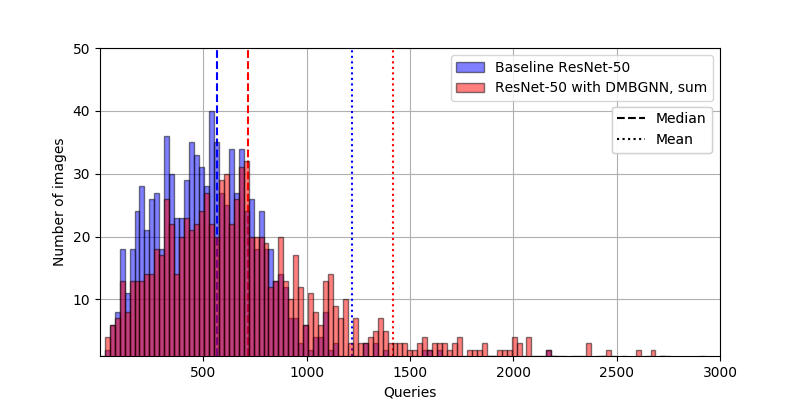}
        \caption{}
    \end{subfigure}
    \hfill
    \begin{subfigure}[b]{0.33\textwidth}
        \centering
        \includegraphics[width=\textwidth]{images/robustness/bandit_dmbgnn_dropfeat.png}
        \caption{}
    \end{subfigure} \\
    \begin{subfigure}[b]{0.33\textwidth}
        \centering
        \includegraphics[width=\textwidth]{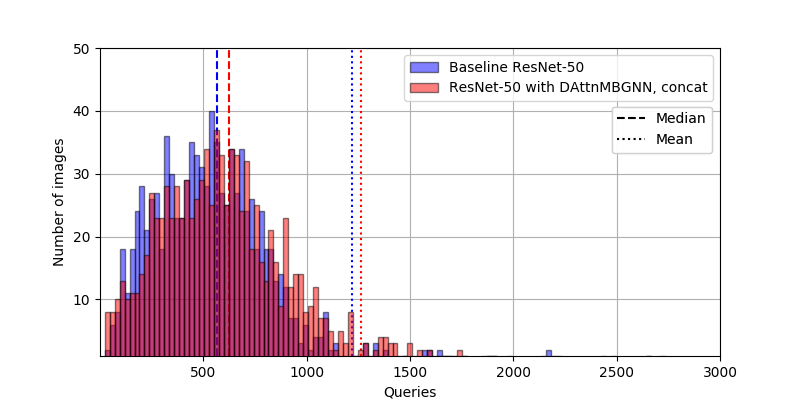}
        \caption{}
    \end{subfigure}
    \hfill
    \begin{subfigure}[b]{0.33\textwidth}
        \centering
        \includegraphics[width=\textwidth]{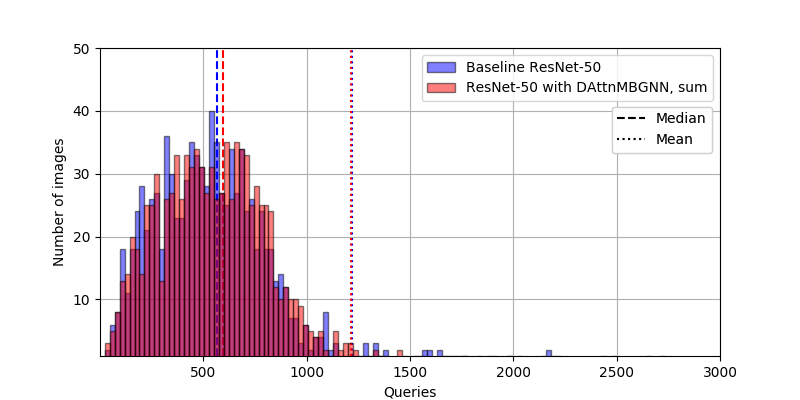}
        \caption{}
    \end{subfigure}
    \hfill
    \begin{subfigure}[b]{0.33\textwidth}
        \centering
        \includegraphics[width=\textwidth]{images/robustness/bandit_dattnmbgnn_dropfeat.png}
        \caption{}
    \end{subfigure}
    \caption{Histogram of number of queries required until a successful attack (over 1000 target images) using Bandits-TD on the CIFAR10 dataset. The queries axis is limited to 3000 queries for clarity of presentation. Models using MBGNNs (red) require a larger number of queries on average for a successful attack, as compared to the baseline model (blue). (a) MBGNN, concat, (b) MBGNN, sum, (c) MBGNN, dropfeat, (d) Attn-MBGNN, concat, (e) Attn-MBGNN, sum, (f) Attn-MBGNN, dropfeat.}
    \label{fig:bandits_query_plots}
\end{center}
\end{figure*}

\clearpage

\begin{figure*}[hbt!]
\begin{center}
    \begin{subfigure}[b]{0.45\textwidth}
        \centering
        \includegraphics[width=\textwidth]{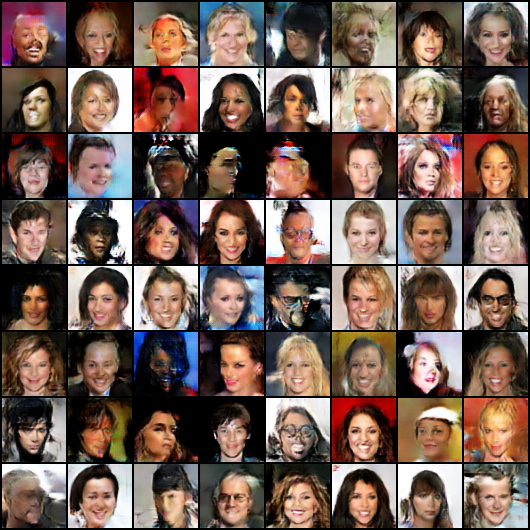}
        \caption{Minibatch Discrimination}
    \end{subfigure}
    \hfill
    \begin{subfigure}[b]{0.45\textwidth}
        \centering
        \includegraphics[width=\textwidth]{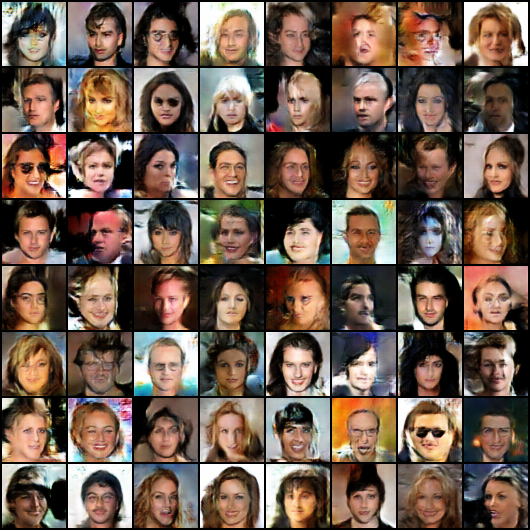}
        \caption{Proposed MC-MBGNN}
    \end{subfigure}
    \caption{Samples generated by the Generator for the CelebA dataset.}
    \label{fig:gen_images}
\end{center}
\end{figure*}

\begin{figure*}[hbt!]
\begin{center}
    \begin{subfigure}[b]{0.45\textwidth}
        \centering
        \includegraphics[width=\textwidth]{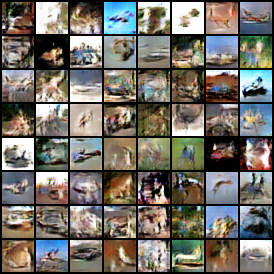}
        \caption{Minibatch Discrimination}
    \end{subfigure}
    \hfill
    \begin{subfigure}[b]{0.45\textwidth}
        \centering
        \includegraphics[width=\textwidth]{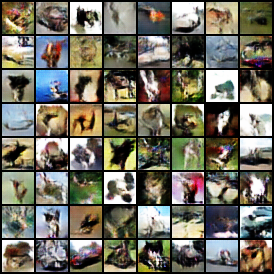}
        \caption{Proposed MC-MBGNN}
    \end{subfigure}
    \caption{Samples generated by the Generator for the CIFAR-10 dataset.}
    \label{fig:gen_images_2}
\end{center}
\end{figure*}

\clearpage

\end{document}